\documentclass[runningheads,orivec]{llncs}

\usepackage{latexsym}
\usepackage{amssymb}
\usepackage{amsmath}
\usepackage{mathrsfs}
\usepackage{qtree}

\renewcommand{\models}{\mid\joinrel=}
\renewcommand{\vdash}{\mid\joinrel\mathrel-}
\renewcommand{\emptyset}{\varnothing}

\newcommand{\shortcite}{\cite}

\title{A Semantic Tableau Method for Argument Construction}

\author{Nico Roos}

\authorrunning{N. Roos}

\institute{Data Science and Knowledge Engineering, Maastricht University
	\email{roos@maastrichtuniversity.nl} \\
	\url{https://dke.maastrichtuniversity.nl/nico.roos/}}

\begin{document}

\maketitle

\begin{abstract}
A semantic tableau method, called an argumentation tableau, that enables the derivation of arguments, is proposed. First, the derivation of arguments for standard propositional and predicate logic is addressed. Next, an extension that enables reasoning with defeasible rules is presented. Finally, reasoning by cases using an argumentation tableau is discussed.

\end{abstract}

\section{Introduction}
The semantic tableau method is used for (automated) reasoning with different logics such as the standard propositional and predicate logic \cite{Beth62}, several modal logics, 
description logics, etc. Although a semantic tableau proof can be viewed as an argument for a claim / conclusion, it is not similar to arguments studied in argumentation systems; see for instance: \cite{BesnardH05,BesnardH09,MP-14,Pol-91,Pol-92,PV-02,Prakken10,Roo-97,Roo-00,Sim-92,Vre-97,YunOC20}. This raises the question whether the semantic tableau method can be used to derive proper arguments for claims / conclusions. 

We will address this question by first investigating a semantic tableau method, called an argumentation tableau, for the derivation of arguments in standard propositional and predicate logic. The use of arguments becomes more interesting when dealing with defeasible information. That is, we consider information that need not be valid in the context of other information. We will address the handling of defeasible information, specifically propositional and predicate logic extended with defeasible rules. Defeasible rules are special rules without contraposition that allow for exceptions in specific situations.

Reasoning by cases is a problem for many argumentation systems that use an underlying language that allows for disjunctive information. Moreover, approaches that support reasoning by cases, do not agree on how rebutting attacks should be handled within a case \cite{BHS-17,Bodanza02,Pol-91,Pol-92,Roo-00}. We also will investigate reasoning by cases using an argumentation tableau. 

The remainder of the paper is organized as follows:
The next section introduces the argumentation system that will be used in the paper. Section \ref{sec:bat} presents the argumentation tableau for standard propositional and predicate logic. Section \ref{sec:defeasible-rules} describes an argumentation tableau for propositional and predicate logic extended with defeasible rules. Section \ref{sec:rbc} discusses reasoning by cases using an argumentation tableau, and Section \ref{sec:conclusion} concludes the paper. 

\section{Preliminaries}
This section presents the notion of an argument that will be used in the discussion of the argumentation tableau that is proposed in this paper. 

We assume a standard logic such as propositional or predicate logic. The language of the logic will be denoted by $\mathcal{L}$. We also assume that the language $\mathcal{L}$ contains the symbols $\top$ denoting \emph{true}, and $\bot$ denoting \emph{false}. In case of predicate logic, the set of ground terms is denoted by $\mathcal{G}$.

Since this paper focuses on argumentation, we need a definition of an argument. Toulmin \shortcite{Tou-58} views an argument as a support for some \emph{claim}. The support is grounded in \emph{data}, and the relation between the data and the claim is the \emph{warrant}. 
Here, we use the following definition.
\begin{definition}
A couple $A = (\mathcal{S},\varphi)$ is called an argument where
$\varphi$ is said to be its conclusion, and $\mathcal{S}$ is a set said to be its support; its elements are called supporting elements. 
It is worthwhile observing here that this definition is very general and a many couples might be qualified as arguments.
\end{definition}
In case of propositional and predicate logic, the support $\mathcal{S}$ is a set of propositions from the language $\mathcal{L}$. Generally, $\mathcal{S}$ contains the set of premises used to derive the supported proposition $\varphi$. So, $\mathcal{S} \vdash \varphi$. In special applications, such as Model-Based Diagnosis, 
we may restrict $\mathcal{S}$ to assumptions about the normal behavior of components.

We may extend a standard logic with a set of defeasible rules. Defeasible rules are of the form:
\[ \varphi \leadsto \psi \]
in case of propositional logic, and of the form:
\[ \varphi(\mathbf{x}) \leadsto \psi(\mathbf{x}) \]
in case of predicate logic. 
Here, $\varphi$ is propositions from the language $\mathcal{L}$, $\psi$ is either a proposition from the language $\mathcal{L}$ or a negated defeasible rule of the form: $\mathbf{not}(\eta \leadsto \mu)$, and $\mathbf{x}$ is a sequence of free variables. 
The free variables denote a set of ground instances of the defeasible rule $\varphi(\mathbf{x}) \leadsto \psi(\mathbf{x})$. We do not use the universal quantifier because the rule is not a proposition that belongs to the language $\mathcal{L}$. It is an additional statement about preferences that need not be valid for every ground instance. 

The defeasible rules $\varphi \leadsto \mathbf{not}(\eta \leadsto \mu)$ and $\varphi(\mathbf{x}) \leadsto \mathbf{not}(\eta(\mathbf{x}) \leadsto \mu(\mathbf{x}))$ are called \emph{undercutting defeaters} \cite{Pol-87}. These undercutting defeaters specify the conditions $\varphi$ and $\varphi(\mathbf{x})$ under which the defeasible rules $\eta \leadsto \mu$ and $\eta(\mathbf{x}) \leadsto \mu(\mathbf{x})$ respectively, are not applicable. 

We use $\Sigma \subseteq \mathcal{L}$ to denote the set of available information and we use $D$ to denote the set of available rules. 
Moreover, we use $\overline{D} = \{ \varphi(\mathbf{t}) \leadsto \psi(\mathbf{t}) \mid \varphi(\mathbf{x}) \leadsto \psi(\mathbf{x}) \in D, \mathbf{t} \in \mathcal{G}^n \}$ to denote the set of ground instances of the defeasible rules with $n$ free variables in case of predicate logic, and $\overline{D}=D$ in case of propositional logic.

Defeasible rules are used in the construction of arguments. Whenever we have a support $\mathcal{S}'$ for the antecedent $\varphi$ of a defeasible rule $\varphi \leadsto \psi$, we can create a supporting element $(\mathcal{S}',\varphi \leadsto \psi)$, which can be used to support $\psi$. The  arguments that can be constructed are defined as:
\begin{definition} \label{def:arguments}
Let $\Sigma \subseteq \mathcal{L}$ be the initial information and let $D$ be a set of defeasible rules. An argument $A = (\mathcal{S},\psi)$ with premises $\bar{A}$, defeasible rules $\tilde{A}$, last defeasible rules $\vec{A}$, supported proposition (claim / conclusion) $\hat{A}$, and supporting propositions $\hat{\mathcal{S}}$ of $\hat{A}$, is recursively defined as:
	\begin{itemize}
		\item 
		If $\psi \in \Sigma$, then $A=(\{\psi\},\psi)$ is an argument.
		
		$\bar{A} = \{\psi\}$. \quad
		$\tilde{A} = \emptyset$. \quad
		$\hat{A} = \psi$. \quad
		$\hat{\mathcal{S}} = \{\psi\}$.
		\item 
		If $A_1=(\mathcal{S}_1,\varphi_1), \ldots, A_k=(\mathcal{S}_k,\varphi_k)$ are arguments and $\{ \varphi_1, \ldots, \varphi_k \} \vdash \psi$, then $A=(\mathcal{S}_1 \cup \cdots \cup \mathcal{S}_k,\psi)$.
		
		$\bar{A} = \bar{A}_1 \cup \cdots \cup \bar{A}_k$. \quad
		$\tilde{A} = \tilde{A}_1 \cup \cdots \cup \tilde{A}_k$. \quad
		$\vec{A} = \vec{A}_1 \cup \cdots \cup \vec{A}_k$. \quad
		$\hat{A} = \psi$. \quad
		$\hat{\mathcal{S}} = \hat{\mathcal{S}}_1 \cup \cdots \cup \hat{\mathcal{S}}_k$.
		\item 
		If $A'=(\mathcal{S}',\varphi)$ is an argument and $\varphi \leadsto \psi \in \overline{D}$ is a defeasible rule, then $A=( \{(\mathcal{S}',\varphi \leadsto \psi)\}, \psi)$ is an argument.
		
		$\bar{A} = \bar{A}'$. \quad
		$\tilde{A} = \{\varphi \leadsto \psi\} \cup \tilde{A}'$. \quad
		$\vec{A} = \{\varphi \leadsto \psi\}$. \quad
		$\hat{A} = \psi$. \quad
		$\hat{\mathcal{S}} = \{\psi\}$.
	\end{itemize}
	$A = (\mathcal{S},\psi)$ is a \emph{minimal} argument iff (1) $\mathcal{S}$ is a minimal set such that $\hat{\mathcal{S}} \vdash \psi$, and (2) for every $(\mathcal{S}',\alpha \leadsto \beta) \in \mathcal{S}$, $(\mathcal{S}',\alpha)$ is a minimal argument.
\end{definition}

This abstract representation of arguments is based on the representation of arguments proposed in \cite{Roo-97,Roo-00}. 
Note that for every argument, there exists a corresponding minimal argument supporting the same conclusion.

We will use a graphical representation of an argument for human readability. The argument for an inconsistency:
\[ A= \begin{array}[t]{l}(\{ (\{ (\{ p \vee q, \neg q \}, p \leadsto r), (\{s\},s \leadsto t) \}, r \wedge t \leadsto u), \\ (\{v\},v \leadsto w), \neg (u \wedge w) \}, \bot) \end{array} \]
is graphically represented as:
\[ A: \left. \begin{array}{r}
\left. \begin{array}{r}
\left. \begin{array}{r} p\vee q \\ \neg q \end{array} \right|\hspace{-5pt}- p \leadsto r \\ \ \\
s \vdash s \leadsto t \end{array} \right|\hspace{-5pt}- r \wedge t \leadsto u \\
v \vdash v \leadsto  w \\
\neg (u \wedge w) \end{array} \right|\hspace{-5pt}- \bot
\]
Here, $\hat{A}=\bot$, $\vec{A}=\{ r \wedge t \leadsto u, v \leadsto w \}$, $\tilde{A} = \{ p\leadsto r, s \leadsto t, r \wedge t \leadsto u, v \leadsto w \}$, $\bar{A}=\{ p\vee q, \neg q, s, v, \neg (u \wedge w) \}$ and $\hat{\mathcal{S}} = \{ u, w, \neg (u \wedge w) \}$ with $A=(\mathcal{S},\bot)$.

When an argument for an inconsistency is derived\footnote{Arguments for inconsistencies cover rebutting attacks.}, one of the defeasible rules is not applicable in the current context. If no defeasible rule is involved in the argument for the inconsistency, one of the premises is invalid. In both cases we will use a strict partial order $<$ on the defeasible rules $D$ and on the information in $\Sigma$ to determine the rule and premise that is invalid, respectively. Following \cite{Roo-89a,Roo-92,Roo-97,Roo-00}, we formulate an \emph{undercutting} argument for the culprit. That is, an argument attacking every argument that uses the culprit.\footnote{Note the difference between an undercutting argument and an undercutting defeater. The former is an argument for not using a proposition or a defeasible rule, and the latter is a defeasible rule specifying a condition under which another defeasible rule should not be used \cite{Pol-87}.}
\begin{definition}
Let $A=(\mathcal{S}, \bot)$ be an argument for an inconsistency. Moreover, let $< \; \subseteq (\Sigma \times \Sigma) \cup (D \times D)$  be a strict partial order over the information $\Sigma$ and over the defeasible rules $D$. Finally, let $A'=(\mathcal{S}', \mathbf{not} (\varphi \leadsto \psi))$ and $A'=(\mathcal{S}', \mathbf{not} (\sigma))$ denote the arguments for an undercutting attack of a defeasible rule in $\overline{D}$ and a proposition in $\Sigma$ respectively.
\begin{itemize}
\item
If $\tilde{A} \not= \emptyset$, \textbf{defeat the weakest last rule}.
For every $\varphi \leadsto \psi \in min_<(\vec{A})$ with $(\mathcal{S}'',\varphi \leadsto \psi ) \in \mathcal{S}$, $A' = (\mathcal{S} \backslash (\mathcal{S}'',\varphi \leadsto \psi ), \mathbf{not} (\varphi \leadsto \psi))$ is an undercutting argument of $\varphi \leadsto \psi \in D$.
\item
If $\tilde{A} = \emptyset$, \textbf{defeat the weakest premise}.
For every $\sigma \in min_<(\bar{A})$, $A' = (\mathcal{S} \backslash \sigma, \mathbf{not} (\sigma))$ is an undercutting argument of $\sigma \in \Sigma$.
\end{itemize}
\end{definition}
Note that $min_<(\cdot)$ need not be unique because $<$ is a strict partial order. Also note that $\mathcal{S} \backslash (\mathcal{S}',\varphi \leadsto \psi )$ is an argument for $\neg \psi$, and that $\mathcal{S} \backslash \sigma$ is an argument for $\neg \sigma$. 

The undercutting arguments define an attack relation over the arguments. We denote the attack relation over a set of arguments $\mathcal{A}$ by $\longrightarrow \ \subseteq \mathcal{A} \times \mathcal{A}$. An undercutting argument $A = (\mathcal{S}, \mathbf{not} (\varphi \leadsto \psi))$ attacks every argument $A'$ for which $\varphi \leadsto \psi \in \tilde{A}'$ holds.  Moreover, an undercutting argument $A = (\mathcal{S}, \mathbf{not}(\sigma) )$ attacks every argument $A'$ for which $\sigma \in \bar{A}'$ holds. We denote the attack of $A$ on $A'$ by $A \longrightarrow A'$. The set of all derived arguments $\mathcal{A}$ and the attack relation over the arguments $\longrightarrow \ \subseteq \mathcal{A} \times \mathcal{A}$ determine an instance of an argumentation framework $(\mathcal{A},\longrightarrow)$ as defined by Dung \cite{Dun-95}.
We can use one the semantics for argumentation frameworks to determine sets of valid arguments; i.e., the argument extensions. See for instance: \cite{Bar-05,Cam-06,CramerT19,Dun-95,Dun-07,DvorakG16,Roo-10,Ver-96}.

\section{Basic Argumentation Tableau} \label{sec:bat}
A semantic tableau method is a proof system developed by Beth \shortcite{Beth62}. In the modern version of the method, the semantic tableau  for propositional and predicate logic is a tree where each node is labeled by a set of propositions. The set of propositions that labels a node of the tree is satisfiable if and only if the set of propositions that labels one of its child nodes, is satisfiable. For convenience we will use $\Gamma$ to denote a node of the semantic tableau as well as the set of propositions that labels the node. 

We are interested in arguments, which are propositions and their supports. Therefore we introduce an \emph{argumentation tableau} of which each node $\Gamma$ is a set of arguments.

\begin{definition}
Let $\mathcal{T}$ be an argumentation tableau. $\mathcal{T}$ is a tree of which each node $\Gamma$ is of a set of arguments.
\end{definition}

The tableau rules of an argumentation tableau are similar to the rules of a traditional semantic tableau. The only difference is the supports for the propositions. In the remainder of the paper, we will focus on the tableaux for propositional and predicate logic. However, the results are not limited to these logic. The approach can also be applied to semantic tableaux for several modal logics \cite{Mas-00}, dynamic logic \cite{BaaderS01}, etc.  
The tableau rules for propositional logic arguments are:
\begin{center}
\begin{tabular}{cc}
$\displaystyle \frac{(\mathcal{S},\varphi \wedge \psi)}{(\mathcal{S},\varphi), (\mathcal{S},\psi)} $ &
$\displaystyle \frac{(\mathcal{S},\varphi \vee \psi)}{(\mathcal{S},\varphi) \mid (\mathcal{S},\psi)} $ \\
\\
$\displaystyle \frac{(\mathcal{S},\varphi \to \psi)}{(\mathcal{S},\neg \varphi) \mid (\mathcal{S},\psi)} $ &
$\displaystyle \frac{(\mathcal{S},\varphi \leftrightarrow \psi)}{(\mathcal{S},\varphi \to
\psi), (\mathcal{S},\psi \to \varphi)} $
\\
\\
$\displaystyle \frac{(\mathcal{S},\neg(\varphi \vee \psi))}{(\mathcal{S},\neg \varphi), (\mathcal{S},\neg\psi)} $ & 
$\displaystyle \frac{(\mathcal{S},\neg(\varphi \wedge \psi))}{(\mathcal{S},\neg\varphi) \mid (\mathcal{S},\neg \psi)} $ \\
\\ 
$\displaystyle \frac{(\mathcal{S},\neg(\varphi \to \psi))}{(\mathcal{S},\varphi), (\mathcal{S},\neg\psi)} $ & 
$\displaystyle \frac{(\mathcal{S},\neg(\varphi \leftrightarrow \psi))}{(\mathcal{S},\neg(\varphi \to \psi)) \mid (\mathcal{S},\neg(\psi \to \varphi))} $
\\
\\
$\displaystyle \frac{(\mathcal{S},\neg\neg\varphi)}{(\mathcal{S},\varphi)} $ &
$\displaystyle \frac{(\mathcal{S},\varphi), (\mathcal{S}',\neg\varphi)}{(\mathcal{S} \cup \mathcal{S}',\bot)} $
\end{tabular}
\end{center}
There are three aspects to note: 
\begin{itemize}
\item
The right rule on the last line specifies the support for the closure of a branch of the semantic tableau,
\item 
More than one support for the closure of a branch may be derived. Here, we are interested in every support for a branch closure. 
\item
For an element $(\mathcal{S},\varphi)$ of a tableau node, unlike an argument defined  by Definition \ref{def:arguments}, $\hat{\mathcal{S}} \vdash \varphi$ need not hold.
\end{itemize}

Four additional tableau rules are used for predicate logic. 

\begin{center}
\begin{tabular}{cc}
$\displaystyle \frac{(\mathcal{S}, \forall x \ \varphi)}{(\mathcal{S},\varphi[^x/_t])} $ &
$\displaystyle \frac{(\mathcal{S}, \exists x \ \varphi)}{(\mathcal{S},\varphi[^x/_c])}$
\\
\\
$\displaystyle \frac{(\mathcal{S},\neg(\forall x \ \varphi))}{(\mathcal{S},\neg \varphi[^x/_c])} $ & 
$\displaystyle \frac{(\mathcal{S},\neg(\exists x \ \varphi))}{(\mathcal{S},\neg\varphi[^x/_t])}$
\end{tabular}
\end{center}
Here, $t$ can be any term that occurs in the current node, and $c$ must be a new constant not yet occurring the current node of the argumentation tableau. Since $t$ can be any term that occurs in the current node, the corresponding rule can be applied more than once for the same proposition.

If an argumentation tableau closes, we can determine the support(s) for the closure.
\begin{definition} \label{def:closure}
Let an argumentation tableau $\mathcal{T}$ with $n$ leaf nodes: $\Lambda_1, \ldots, \Lambda_n$. 
\begin{itemize}
\item
The argumentation tableau is \emph{closed} iff for every leaf $\Lambda_i$ there is an argument $(\mathcal{S}_i,\bot) \in \Lambda_i$.
\item
A support for a tableau closure is defined as: \\ $\mathcal{S} = \bigcup_{i=1}^n \mathcal{S}_i$ where $(\mathcal{S}_i,\bot) \in \Lambda_i$.
\end{itemize}
\end{definition}
Note that a leaf of a closed tableau may contain more than one argument of the form $(\mathcal{S}',\bot)$. Therefore, there can be multiple supports for the closure of the tableau. In order to determine every possible $(\mathcal{S}',\bot)$, the leafs of the closed tableau must also be saturated. A leaf node is \emph{saturated} if and only if there are no tableau rules that can be applied. It may be impossible to determine saturated leafs in case of predicate logic.

\begin{proposition} \label{prop:tableau-closure}
Let $\mathcal{L}$ be the language of propositional or predicate logic, let the $\Sigma \subseteq \mathcal{L}$, and let $\mathcal{T}$ be an argumentation tableau. 
Then, 
\begin{enumerate}
	\item
	If $\mathcal{S}$ is a support for the closure of the tableau $\mathcal{T}$ with root node $\Gamma_0 = \{ (\{ \sigma \}, \sigma) \mid \sigma \in \Sigma \}$, then $\mathcal{S} \subseteq \Sigma$ is inconsistent.
	\item
	If $\mathcal{S} \subseteq \Sigma$ is a minimal inconsistent set, then there exists a tableau $\mathcal{T}'$ which extends the tableau $\mathcal{T}$ and $\mathcal{S}$ is a support for the closure of $\mathcal{T}'$.
\end{enumerate}
\end{proposition}
\begin{proof} 
	We can prove that an interpretation entails the root of the tableau iff it entails all nodes on a branch from the root to a leaf. The proof is similar to the proof for a standard semantic tableau. We only have an argument $(\mathcal{S}_i,\bot)$ in a leaf node $\Lambda_i$ iff the branch containing $\Lambda_i$ closes. Therefore, the argumentation tableau is \emph{closed} iff for every leaf $\Lambda_i$ there is an argument $(\mathcal{S}_i,\bot) \in \Lambda_i$.
	
	\begin{enumerate}
		\item
		Let $\mathcal{S}$ be the support of the closure of the tableau $\mathcal{T}$. We can remove from every node, the arguments $(\mathcal{S}',\sigma)$ such that $\mathcal{S}' \not\subseteq \mathcal{S}$.
		This may result in some nodes $\Gamma$ having children that are all equal to $\Gamma$. 
		The following holds for the resulting tableau $\mathcal{T}'$: 
		\begin{itemize}
			\item 
			The tableau $\mathcal{T}'$ has a root $\Gamma'_0 = \{ (\{\sigma\}, \sigma) \mid \sigma \in \mathcal{S} \}$.
			\item 
			There is an interpretation entailing the root iff it entails every node on a branch from the root to a leaf.
			\item 
			The tableau $\mathcal{T}'$ still closes with $\mathcal{S}$ being the support of the closure.
		\end{itemize}
		Hence, $\mathcal{S}$ is an inconsistent set of propositions.
		
		\item 
		Let $\mathcal{S}$ be a minimal inconsistent subset of $\Sigma$. Then there exists a finite argumentation tableau $\mathcal{T}''$ that closes. We can extend the tableau $\mathcal{T}$ by replacing every leaf $\Lambda$ of $\mathcal{T}$ by $\mathcal{T}''$ after adding $\Lambda$ of every node in $\mathcal{T}''$. A rewriting step can occur twice in a branch of the resulting tableau. Since we normally do not have duplicate rewriting steps, we can remove the duplicate rewriting steps in $\mathcal{T}''$, and if a rewriting step resulted in two or more children, we can remove all branches except one. The resulting tableau is $\mathcal{T}'$.
		
		Since the tableau $\mathcal{T}''$ is closed, so is $\mathcal{T}'$. Next, we remove from every node, the arguments $(\mathcal{S}'',\sigma)$ such that $\mathcal{S}'' \not\subseteq \mathcal{S}$.
		This will result in some nodes $\Gamma$ having children that are all equal to $\Gamma$. Clearly, $\mathcal{T}'$ is still closed because of the extension of every leaf with $\mathcal{T}''$. The support $\mathcal{S}'$ of the closure satisfies $\mathcal{S}' \subseteq \mathcal{S}$. Since $\mathcal{S}$ is a minimal inconsistent set, according to the first item of his proposition, $\mathcal{S}' = \mathcal{S}$.
		\hspace*{\fill} $\Box$
	\end{enumerate}
\end{proof}

A standard semantic tableau uses refutation to prove a conclusion. The support $\mathcal{S}$ for a closure of an argumentation tableau can be used for the same purpose. Since $\mathcal{S}$ is inconsistent for any $\sigma \in \mathcal{S}$, $\mathcal{S} \backslash \sigma \models \neg\sigma$. Hence, to prove $\varphi$ and identify a corresponding argument, we add $(\{\neg\varphi\}, \neg\varphi)$ to the root $\Gamma_0$ of the tableau. If the tableau closes and if the support $\mathcal{S}$ of an inconsistency contains $\neg\varphi$, then we can construct an argument  $(\mathcal{S} \backslash \neg\varphi, \varphi)$. 
To keep track of the proposition we try to refute, we put a question-mark behind the proposition in the support $(\{\neg\varphi?\}, \neg\varphi)$. The element $ (\{\neg\varphi?\}, \neg\varphi)$ that we add to the root node, is called a \emph{test}. It is a special supporting element, which has not effect on the application of the tableau rules.
\begin{corollary}
	Let $\Sigma \subseteq \mathcal{L}$ be the initial information and let $\varphi \in \mathcal{L}$ be the proposition for which we search supporting arguments.
	Moreover, let $\mathcal{S}$ be the support for a tableau closure of a tableau $\mathcal{T}$ with root $\Gamma_0 = \{ (\{\sigma\}, \sigma) \mid \sigma \in \Sigma \} \cup (\{\neg\varphi?\}, \neg\varphi)$. 
	\begin{enumerate}
	\item
	If $\mathcal{S}$ is the support for a tableau closure of a tableau $\mathcal{T}$ and $\mathcal{S}$ contains a single test $\neg\varphi?$, then  $\mathcal{S} \backslash \neg\varphi? \vdash \varphi$. 
	\item
	If $\mathcal{S}' \subseteq \Sigma$ is a minimal set such that $\mathcal{S}' \vdash \varphi$, then there exists a tableau $\mathcal{T}'$ which extends the tableau $\mathcal{T}$ and $\mathcal{S}' \cup \{ \varphi? \}$ is a support of its closure.
	\end{enumerate}
\end{corollary}

It can be beneficial if we can derive multiple conclusions simultaneously. The argumentation tableau offers this possibility by simply adding several tests to the root node. After deriving a support $\mathcal{S}$ for a tableau closure, we check whether the support contains multiple tests. If it does, it can be ignored. We are interested in supports with zero or one test. The latter provides arguments for conclusion of interest, and the former enables us to handle with inconsistencies in the initial information. For instance Roos \cite{Roo-88a,Roo-89a,Roo-92} proposes to resolve the inconsistencies by formulating undercutting arguments for the least preferred propositions in $\mathcal{S}$ given a partial preference order $<$ (which can be empty). 
\begin{definition}
	Let $\mathcal{S}$ be a support without tests for the tableau closure.

	For every $\sigma \in \min_< \mathcal{S}$, $(\mathcal{S} \backslash \sigma,\mathbf{not} \ \sigma)$ is an undercutting argument of $\sigma$. 
\end{definition}
Other names that can be found in the literature for this form of undercutting attack are: \emph{premise attack} and \emph{undermining} \cite{Prakken10}. The derivation of arguments for conclusions and undercutting arguments to resolve inconsistencies is related to \cite{BesnardH05,BesnardH09,DungKT09,Roo-92,Toni14}.

\section{Defeasible Rules} \label{sec:defeasible-rules}
The argumentation tableau presented in the previous section enables us to derive deductive arguments. It does not support arguments containing defeasible rules. Here, we will extend the argumentation tableau in order to derive arguments as defined in Definition \ref{def:arguments}. 

The support of the argument defined in Definition \ref{def:arguments} is a tree consisting of alternating deductive and defeasible transitions. The root of the tree is the conclusion / claim supported by the argument. For instance,
\[ A: \left. \begin{array}{r} p\vee q \\ \neg q \end{array} \right|\hspace{-5pt}- p \leadsto r \vdash r \leadsto s \vdash s \]

The support of the deductive transitions can be determined by the basic argumentation tableau described in the previous section by adding the antecedent of a defeasible rule as a test to the root of the argumentation tableau. Since we do not know which antecedents of defeasible rules will be supported, we add all of them as tests to the root $\Gamma_0$. 

Next, we extend every node of the tableau with the consequent of a defeasible rule after determining a support for its antecedent from a tableau closure. In the graphical representation of a tableau in Figures \ref{fig:arg-tableau} and \ref{fig:arg-tableau2}, this corresponds to extending the root of the tableau with the consequent of a defeasible rule after determining a support for its antecedent from a tableau closure. 
\begin{definition}
Let $\mathcal{T}$ be a tableau with root $\Gamma_0$.
Moreover, let $\mathcal{S}$ be the support for the antecedent $\varphi$ of the rule $\varphi \leadsto \psi \in \overline{D}$ determined by the tableau $\mathcal{T}$ where $(\{ \neg\varphi? \},\neg\varphi) \in \Gamma_0$. 

Then we extend every node $\Gamma$ of $\mathcal{T}$ with the argument $(\{ (\mathcal{S}, \varphi \leadsto \psi) \}, \psi)$.
\end{definition}
To give an illustration, consider the initial information $\Sigma = \{ p \vee q, \neg q \}$ and the defeasible rules $D = \{ p \leadsto r, r \leadsto s \}$. We are interested in an argument for the conclusion $s$. We start constructing the tableau shown on the left in Figure \ref{fig:arg-tableau}.
The support for the closure of this tableau is: $\{ p \vee q, \neg q, \neg p? \}$ implying the argument $(\{ p \vee q, \neg q \}, p)$. We may therefore add the consequence $r$ of the defeasible rule $p \leadsto r$ with the support $\{ (\{ p \vee q, \neg q \}, p \leadsto r) \}$ to the root of the tableau and continue rewriting the tableau. This results in the tableau shown on the right in Figure \ref{fig:arg-tableau}.

\begin{figure}
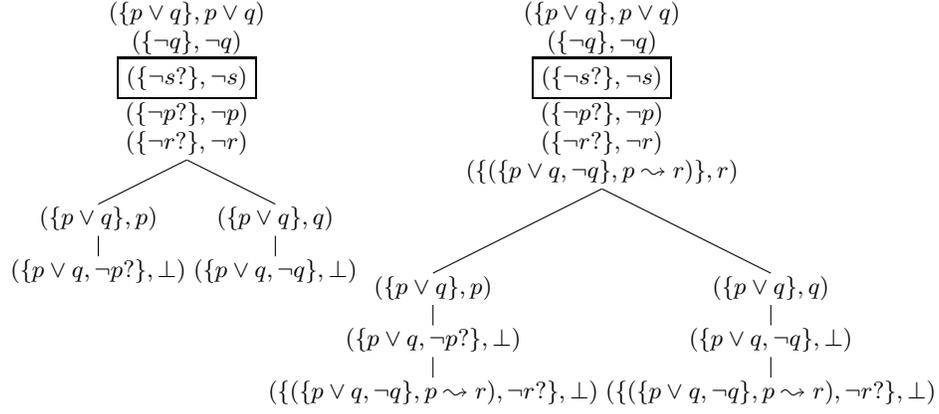

\begin{tabular}{cc}
\mbox{
\hspace*{-4mm}
\footnotesize
 \Tree
 [.{$(\{ p \vee q \}, p \vee q)$ \\
    $(\{ \neg q \}, \neg q)$ \\
    \fbox{$(\{\neg s? \}, \neg s)$} \\
    $(\{\neg p? \}, \neg p)$ \\
    $(\{\neg r? \}, \neg r)$}
    [.{$(\{ p \vee q \}, p )$}
       [.{$(\{ p \vee q, \neg p? \}, \bot )$}
       ]
    ]
    [.{$(\{ p \vee q \}, q)$}
       [.{$(\{ p \vee q, \neg q \}, \bot )$}
       ]
    ]
 ] }
&
\mbox{
\hspace*{-15mm}
\footnotesize
\Tree
 [.{$(\{ p \vee q \}, p \vee q)$ \\
    $(\{ \neg q \}, \neg q)$ \\
    \fbox{$(\{\neg s? \}, \neg s)$} \\
    $(\{\neg p? \}, \neg p)$ \\
    $(\{\neg r? \}, \neg r)$ \\
    $(\{ (\{ p \vee q, \neg q \}, p \leadsto r) \}, r)$}
    [.{$(\{ p \vee q \}, p )$}
       [.{$(\{ p \vee q, \neg p? \}, \bot )$}
          [.{$(\{ (\{ p \vee q, \neg q \}, p \leadsto r), \neg r? \}, \bot)$}
          ]
       ]
    ]
    [.{$(\{ p \vee q \}, q)$}
       [.{$(\{ p \vee q, \neg q \}, \bot )$}
          [.{$(\{ (\{ p \vee q, \neg q \}, p \leadsto r), \neg r? \}, \bot)$}
          ]
       ]
    ]
 ] }
\end{tabular}
\normalsize
\caption{Deriving defeasible arguments 1}
\label{fig:arg-tableau}
\end{figure}
 
\noindent
The support for the new closure of the tableau shown on the right in Figure  \ref{fig:arg-tableau} is: $\{ (\{ p \vee q, \neg q \}, p \leadsto r), \neg r? \}$ implying the argument $(\{ (\{ p \vee q, \neg q \}, p \leadsto r) \}, r)$. We may therefore add the consequence $s$ of the defeasible rule $r \leadsto s$ with the support $\{ (\{ (\{ p \vee q, \neg q \}, p \leadsto r) \}, r \leadsto s) \}$ to the root of the tableau and continue rewriting the resulting tableau as shown in Figure \ref{fig:arg-tableau2}.  The support for the closure of the tableau as shown in Figure \ref{fig:arg-tableau2} is: 
\[ \{ (\{ (\{ p \vee q, \neg q \}, p \leadsto r) \}, r \leadsto s), \neg s? \} \] 
implying the desired argument:
\[ (\{ (\{ (\{ p \vee q, \neg q \}, p \leadsto r) \}, r \leadsto s) \}, s) \]

\begin{figure}
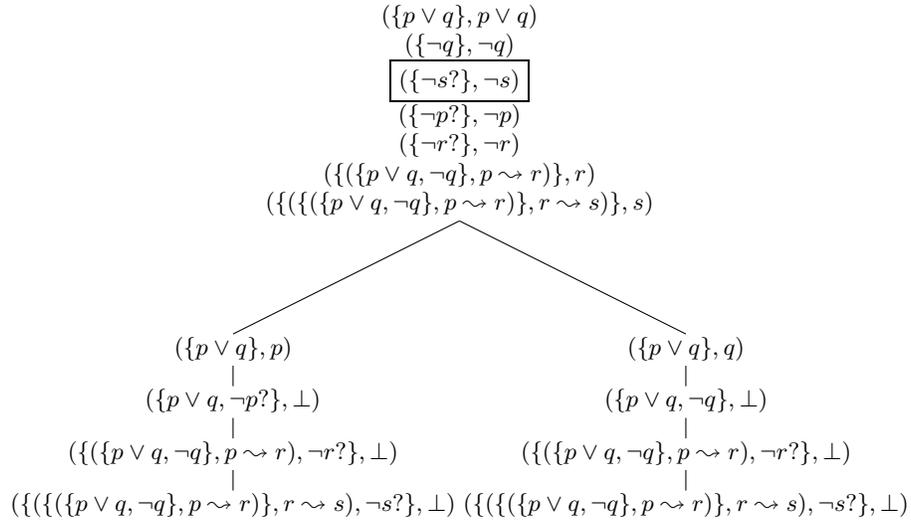

\hspace*{-9mm}
\footnotesize
\hspace*{5mm}
\Tree
 [.{$(\{ p \vee q \}, p \vee q)$ \\
    $(\{ \neg q \}, \neg q)$ \\
    \fbox{$(\{\neg s? \}, \neg s)$} \\
    $(\{\neg p? \}, \neg p)$ \\
    $(\{\neg r? \}, \neg r)$ \\
    $(\{ (\{ p \vee q, \neg q \}, p \leadsto r) \}, r)$ \\
    $(\{ (\{ (\{ p \vee q, \neg q \}, p \leadsto r) \}, r \leadsto s) \}, s)$}
    [.{$(\{ p \vee q \}, p )$}
       [.{$(\{ p \vee q, \neg p? \}, \bot )$}
          [.{$(\{ (\{ p \vee q, \neg q \}, p \leadsto r), \neg r? \}, \bot)$}
             [.{$(\{ (\{ (\{ p \vee q, \neg q \}, p \leadsto r) \}, r \leadsto s), \neg s? \}, \bot)$}
             ]
          ]
       ]
    ]
    [.{$(\{ p \vee q \}, q)$}
       [.{$(\{ p \vee q, \neg q \}, \bot )$}
          [.{$(\{ (\{ p \vee q, \neg q \}, p \leadsto r), \neg r? \}, \bot)$}
             [.{$(\{ (\{ (\{ p \vee q, \neg q \}, p \leadsto r) \}, r \leadsto s), \neg s? \}, \bot)$}
             ]
          ]
       ]
    ]
 ] 
\normalsize
\caption{Deriving defeasible arguments 2}
\label{fig:arg-tableau2}
\end{figure}

\subsection{Predicate Logic}
The construction of an argumentation tableau for predicate logic extended with defeasible rules is the same as the above described argumentation tableau for propositional logic with defeasible rules. We should in principle add every ground instance of the negated antecedent of each rule $\varphi(\mathbf{t}) \leadsto \psi(\mathbf{t}) \in \overline{D}$ as a test to the root of the tableau. That is, we should add the set of tests 
\[ \{ ( \{ \neg \varphi(\mathbf{t})? \}, \neg \varphi(\mathbf{t})) \mid \varphi(\mathbf{t}) \leadsto \psi(\mathbf{t}) \in \overline{D}, \mathbf{t} \in \mathcal{G} \}\] 
to the root of the tableau.  If functions are used, this set of tests will be infinite, and therefore adding all ground instances is not practically feasible. Instead, we may limit ourselves to the ground instances that are present in the current tableau. So, while expanding the tableau, more ground instance may be added.

\subsection{Correctness and Completeness}
We can proof that the argumentation tableau determines exactly the same set of arguments as those defined in Definition \ref{def:arguments}. First, we prove a proposition similar to Proposition \ref{prop:tableau-closure}
\begin{proposition} \label{prop:tableau-with-rules-closure}
	Let $\mathcal{L}$ be the language of propositional or predicate logic, let the $\Sigma \subseteq \mathcal{L}$, let $D$ be a set of defeasible rules over $\mathcal{L}$, and let $\Gamma_0 = \{ (S_i,\psi_i) \}^n_{i=1}$ be the root node of the tableau $\mathcal{T}$ and let $\Psi = \{ \psi \mid (S,\psi) \in \Gamma_0 \}$. 
	Then, 
	\begin{enumerate}
		\item
		If $\mathcal{S}$ is a support for the closure of the tableau $\mathcal{T}$, then $\hat{\mathcal{S}} \subseteq \Psi$ is inconsistent.
		\item
		If $\hat{\mathcal{S}} \subseteq \Psi$ is a minimal inconsistent set, then $\mathcal{S}$ is a support for the closure of the tableau $\mathcal{T}$.
	\end{enumerate}
\end{proposition}
\begin{proof}
	Since $\hat{\mathcal{S}}$ is a subset of $\mathcal{L}$, the proof is similar to the proof of Proposition \ref{prop:tableau-closure}.
	\hspace*{\fill} $\Box$
\end{proof}

\begin{theorem}
If $A$ is a minimal argument according to Definition \ref{def:arguments}, then $A$ can be derived by an argumentation tableau.
If the argument $A$ can be derived by an argumentation tableau, then $A$ is an argument according to Definition \ref{def:arguments}.
\end{theorem}

\begin{proof}
	We prove the theorem by induction on the construction of an argument. 
	
	\noindent\textit{Initialization step:} 
	Let $\sigma \in \Sigma$. Clearly, $A=(\{\sigma\},\sigma)$ is an argument according to Definition \ref{def:arguments} iff the tableau with test $(\{ \neg \sigma?\},\neg \sigma)$ closes with support $\{ \sigma, \neg \sigma? \}$.
	
	\noindent\textit{Induction step:}
	\begin{itemize}
		\item 
		Let $A = (\mathcal{S}, \varphi)$ with $\hat{\mathcal{S}} \vdash \varphi$ be a minimal argument according to Definition \ref{def:arguments}. Then, $\hat{\mathcal{S}} \cup \{ \neg \varphi \}$ is a minimal inconsistent set. Therefore, according to Proposition \ref{prop:tableau-with-rules-closure}, $\mathcal{S} \cup \{(\{\neg\varphi?\},\neg\varphi)\}$ supports a tableau closure, and $A = (\mathcal{S}, \varphi)$ can be derived by an argumentation tableau. 
		
		Let $A = (\mathcal{S}, \varphi)$ be an argument that can be derived by an argumentation tableau. Then $\mathcal{S} \cup \{(\{\neg\varphi?\},\neg\varphi)\}$ supports a tableau closure, and according to Proposition \ref{prop:tableau-with-rules-closure}, $\hat{\mathcal{S}} \vdash \varphi$. So, $A = (\mathcal{S}, \varphi)$ is an argument according to Definition \ref{def:arguments}.
		\item
		Let $A = (\{(\mathcal{S},\varphi \leadsto \psi) \}, \psi)$ be a minimal argument according to Definition \ref{def:arguments}. Then $\hat{\mathcal{S}} \vdash \varphi$ and there exists an argument $A' = (\mathcal{S}, \varphi)$. According to the previous item, $\mathcal{S} \cup \{(\{\neg\varphi?\},\neg\varphi)\}$ supports a tableau closure. Therefore, $(\{(\mathcal{S},\varphi \leadsto \psi) \}, \psi)$ can be added to the root of the tableau. Hence,  $A = (\{(\mathcal{S},\varphi \leadsto \psi) \}, \psi)$ is an argument that can be derived by an argumentation tableau.
		
		Let $A = (\{(\mathcal{S},\varphi \leadsto \psi) \}, \psi)$ be an argument that can be derived by an argumentation tableau. Then $\mathcal{S} \cup \{(\{\neg\varphi?\},\neg\varphi)\}$ supports a tableau closure. So, $\hat{\mathcal{S}} \vdash \varphi$, and $A = (\{(\mathcal{S},\varphi \leadsto \psi) \}, \psi)$ is an argument according to Definition \ref{def:arguments}. 
		\hspace*{\fill} $\Box$
	\end{itemize}
\end{proof}

\section{Reasoning by Cases} \label{sec:rbc}
Reasoning by cases addresses the derivation of conclusions in the context of uncertainty. Uncertainty described by disjunctions results in multiple cases. Each case is a possible description of the world. If the same conclusion is derived in each case, then that conclusion will certainly hold in the case describing the world. The use of defeasible rules to derive new conclusions in a case should make no difference despite that the arguments supporting the conclusions might defeat other arguments. 

\subsection{Cases in an argumentation tableau}
If we ignore the \emph{tests} that we add to the root of an argumentation tableau, then the construction of a tableau can be viewed as the construction of all cases implied by the available information. Ignoring the tests, each open branch describes one case implied by the available disjunctive information. If a case describes the world, additional information may eliminate all other cases and a defeasible rule should be applied as described in the previous section.

The use of defeasible rules in a case implies that we should extend a leaf of the argumentation tableau with the consequence of a defeasible rule whenever the leaf entails the antecedent this rule. We cannot test whether a leaf entails the antecedent of a defeasible rule by adding the antecedent as a test to the root of the tableau. We should add the antecedent to the leaf. Preferably the leaf is saturated because a possibly successful test may fail if we add it too early. To give an illustration, consider $\Sigma = \{ p \vee q \}$ and $D = \{ p \leadsto r, q \leadsto r \}$. If we add the tests $(\{\neg p?\},\neg p)$ and $(\{\neg q?\},\neg q)$ to the root of the tableau, both tests will fail because there is no support for a tableau closure with only one test. If however we first rewrite $p \vee q$ and then add the tests to the resulting leafs, in each branch we will derive a support for a closure that enables us to add the consequence of the corresponding rule. The two cases are illustrated by the two tableaux in Figure \ref{fig:r.b.c}.

\begin{figure}
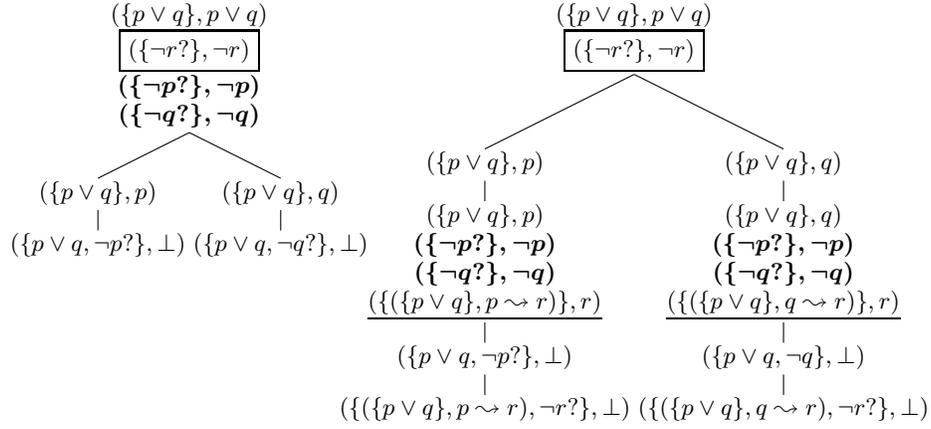

\begin{tabular}{cc}
\mbox{
\footnotesize
\hspace*{-3mm}
\Tree
[.{$(\{ p \vee q \}, p \vee q)$ \\
	\fbox{$(\{\neg r? \}, \neg r)$} \\
	$\boldsymbol{(\{\neg p? \}, \neg p)}$ \\
	$\boldsymbol{(\{\neg q? \}, \neg q)}$}
	[.{$(\{ p \vee q \}, p )$}
		[.{$(\{ p \vee q, \neg p? \}, \bot )$}
		]
	]
	[.{$(\{ p \vee q \}, q)$}
		[.{$(\{ p \vee q, \neg q? \}, \bot)$}
		]
	]
] }
&
\mbox{
\footnotesize
\hspace*{-7mm}
\Tree
[.{$(\{ p \vee q \}, p \vee q)$ \\
	\fbox{$(\{\neg r? \}, \neg r)$} }
	[.{$(\{ p \vee q \}, p )$}
		[.{$(\{ p \vee q \}, p )$ \\
			$\boldsymbol{(\{\neg p? \}, \neg p)}$ \\
			$\boldsymbol{(\{\neg q? \}, \neg q)}$ \\
			$\underline{(\{ (\{ p \vee q \},p \leadsto r) \}, r )}$}
			[.{$(\{ p \vee q, \neg p? \}, \bot )$}
				[.{$(\{ (\{ p \vee q \},p \leadsto r), \neg r? \}, \bot )$}
				]
			]
		]
	]
	[.{$(\{ p \vee q \}, q )$}
		[.{$(\{ p \vee q \}, q )$ \\
			$\boldsymbol{(\{\neg p? \}, \neg p)}$ \\
			$\boldsymbol{(\{\neg q? \}, \neg q)}$ \\
			$\underline{(\{ (\{ p \vee q \},q \leadsto r) \}, r )}$}
			[.{$(\{ p \vee q, \neg q \}, \bot )$}
				[.{$(\{ (\{ p \vee q \},q \leadsto r), \neg r? \}, \bot )$}
				]
			]
		]
	]
] }
\end{tabular}
\normalsize
\caption{Reasoning by cases.}
\label{fig:r.b.c}
\end{figure}

The example illustrates that adding the tests is a strategic choice, which can be dealt with through search. We add a test for the negated antecedent of a rule to a current leaf and try to close all  resulting branches starting from the leaf. If we cannot close all these branches, we backtrack to the leaf and remove the test. Using such a search process is of course not a very efficient solution.

Instead of adding tests for the antecedents of defeasible rules, we can check whether the current leaf of a branch of a tableau entails the antecedent. This works fine for propositional logic but raises a problem for predicate logic. If the antecedent of a rule contains a universal claim; i.e., a universally quantified proposition that must be true or an existentially quantified proposition that must be false,  then entailment is not decidable because we do not know all the objects over which we have to quantify. So, we should restrict the defeasible rules to those that do not contain universal claims in the antecedent. This restriction implies that we cannot state that \emph{a Student that Passes all Exams normally receives a Diploma}: $S(x) \wedge \forall [E(y) \to P(x,y)] \leadsto D(x)$. This even holds if the exams have been specified explicitly: $\forall y [E(y) \leftrightarrow y =e_1 \vee \cdots \vee y=e_n]$.

A possible solution for this restriction is a first order logic that uses binary quantifiers in combination with a special specification of the ground terms for which a predicate is true: $E = \{ e_1 \ldots e_n\}$ and $S(x) \wedge \forall E(y) [P(x,y)] \leadsto D(x)$.
However, if we wish to stay in the domain of standard predicate logic, we should rely on the above described search process. 

\subsection{How to reason by cases with defeasible information}
There have been a few proposals how to introduce reasoning by cases in argumentation systems \cite{BHS-17,Bodanza02,Pol-92,Roo-00}. Unfortunately, there is no consensus on the correct conclusion(s) when reasoning by cases using defeasible information. Here, we propose that \emph{the (defeasible) conclusions supported in a case by defeasible information must be the same as when uncertainty is eliminated by additional information}. This principle implies that we only eliminate alternative cases (through additional information) in which the antecedent of a defeasible rule does not hold. Note that a case can therefore have sub-cases. To give an illustration, consider the information $\Sigma = \{ \neg(p \wedge q), r \vee s, t \}$ and the  defeasible rules $D = \{ r \leadsto p, t \leadsto q \}$. The defeasible rule $r \leadsto p$ is applicable in the case $\{ \neg(p \wedge q), r, t \}$. This case has two sub-cases, $\{ \neg p, r, t \}$ and $\{ \neg q, r, t \}$. An inconsistency can be derived in the case $\{ \neg(p \wedge q), r, t \}$ and the set of last rules involved in the inconsistency is: $\{ r \leadsto p, t \leadsto q \}$.


Before addressing the technical details of reasoning be cases in using an argumentation tableau, we will first briefly review proposals made in the literature.
\begin{itemize}
\item
Pollock's argumentation system OSCAR \shortcite{Pol-91,Pol-92}
is an example of an argumentation system that allows for
suppositional reasoning, and is therefore capable of reasoning by
cases. Pollock does not explicitly discuss which conclusions
should be supported when using reasoning by cases with defeasible
rules. His definition of rebutting attack \cite{Pol-91} implies
that a suppositional argument can only
be defeated by (1) suppositional arguments of the same case, and (2) by arguments that do not depend on the considered cases. A suppositional argument cannot defeat an argument that does not depend on any case. 
As argued in \cite{Roo-00}, this restriction may result in incorrect conclusions.
\item
Bodanza \shortcite{Bodanza02} adapts OSCAR by allowing that a suppositional argument defeats an argument that does not depend on any case. However, Bodanza  changes the interpretation of the $\neg$-operator. $\neg \alpha$ is interpreted as: ``$\alpha$ is not an alternative'' when reasoning by cases.
\item
Recently, the framework for structured argumentation ASPIC$^+$ \cite{MP-14,Prakken10} has been
extended in order to enable reasoning by cases \cite{BHS-17}. The
authors introduce hypothetical sub-arguments to handle the cases.
An argument can attack a hypothetical sub-argument but not vice
versa. Hypothetical sub-arguments can only attack other
hypothetical sub-arguments. 
\end{itemize}

The first and the last approach above result in counter-intuitive conclusions in the following example. 
\begin{quote}
Harry and Draco are involved in a fight and therefore are punishable.
However, if someone involved in a fight, acted in self-defense,
then he or she is not punishable. Witnesses state that either
Harry or Draco acted in self-defense. 
\end{quote}
The first and last approach above support the conclusion that both Harry and Draco are punishable, while we would expect that only one of them is punishable.
Our proposal that \emph{conclusions supported in a case by defeasible information must be the same as when uncertainty is eliminated by additional information} avoids the counter-intuitive conclusion. However, it introduces a technical issue, which will be discussed in the next subsection.

\subsection{Local tableau closures}
Reconsider the above example with information $\Sigma = \{ \neg(p \wedge q), r \vee s, t \}$ and  defeasible rules $D = \{ r \leadsto p, t \leadsto q \}$. We can use the information and the rules to construct the tableau in Figure \ref{fig:local-closure}. If we eliminate the right most branch by adding the information $\neg s$, we get a tableau as described in Section \ref{sec:defeasible-rules}, and the set of last rules for the derived inconsistency is: $\{ r \leadsto p, t \leadsto q \}$. It is not difficult to determine the same inconsistency in the tableau in Figure \ref{fig:local-closure}.

\begin{figure}
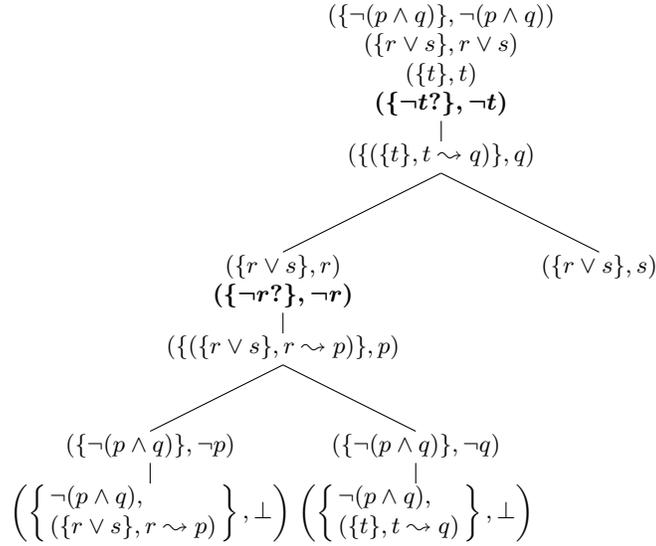

\begin{center}
\mbox{
\footnotesize
\Tree
[.{$(\{\neg(p \wedge q)\},\neg(p \wedge q))$ \\
	$(\{ r \vee s \}, r \vee s)$ \\
	$(\{ t \}, t)$ \\
	$\boldsymbol{(\{\neg t? \}, \neg t)}$}
	[.{$(\{ (\{ t \}, t \leadsto q) \}, q)$}
		[.{$(\{ r \vee s \}, r)$ \\
			$\boldsymbol{(\{\neg r? \}, \neg r)}$}
			[.{$(\{ (\{ r \vee s \}, r \leadsto p) \}, p)$}
				[.{$(\{\neg(p \wedge q)\},\neg p)$}
					[.{$\left(\left\{ \begin{array}{l} \neg(p \wedge q), \\ (\{ r \vee s \}, r \leadsto p) \end{array} \right\}, \bot \right)$}
					]
				]
				[.{$(\{\neg(p \wedge q)\},\neg q)$}
					[.{$\left(\left\{ \begin{array}{l} \neg(p \wedge q), \\ (\{ t \}, t \leadsto q) \end{array} \right\}, \bot \right)$}
					]
				]
			]
		]
		[.{$(\{ r \vee s \}, s)$}
		]
	]
] }
\end{center}
\normalsize
\caption{Local tableau closure 1.}
\label{fig:local-closure}
\end{figure}

It is also possible to construct the tableau in Figure \ref{fig:local-closure2} using the same information. Here, it is more difficult to determine the set of last rules involved in the inconsistent case.

\begin{figure}
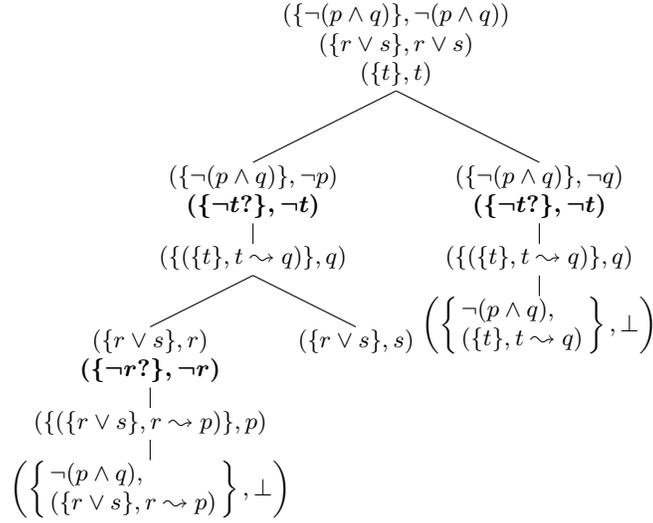

\begin{center}
\mbox{
\footnotesize
\Tree
[.{$(\{\neg(p \wedge q)\},\neg(p \wedge q))$ \\
	$(\{ r \vee s \}, r \vee s)$ \\
	$(\{ t \}, t)$}
	[.{$(\{\neg(p \wedge q)\},\neg p)$ \\
		$\boldsymbol{(\{\neg t? \}, \neg t)}$}
		[.{$(\{ (\{ t \}, t \leadsto q) \}, q)$}
			[.{$(\{ r \vee s \}, r)$ \\
				{$\boldsymbol{(\{\neg r? \}, \neg r)}$}}
				[.{$(\{ (\{ r \vee s \}, r \leadsto p) \}, p)$}
					[.{$\left(\left\{ \begin{array}{l} \neg(p \wedge q), \\ (\{ r \vee s \}, r \leadsto p) \end{array} \right\}, \bot \right)$}
					]
				]
			]
			[.{$(\{ r \vee s \}, s)$}
			]
		]
	]
	[.{$(\{\neg(p \wedge q)\},\neg q)$ \\
		$\boldsymbol{(\{\neg t? \}, \neg t)}$}
		[.{$(\{ (\{ t \}, t \leadsto q) \}, q)$}
			[.{$\left(\left\{ \begin{array}{l} \neg(p \wedge q), \\ (\{ t \}, t \leadsto q) \end{array} \right\}, \bot \right)$}
			]
		]
	]
] }
\end{center}
\normalsize
\caption{Local tableau closure 2.}
\label{fig:local-closure2}
\end{figure}

The proposition $r \vee s$ in the above example specifies two cases: $r$ and $s$. If we eliminates the case $s$, we get an argumentation tableau as described in the previous section. Eliminating the case corresponds to eliminating the right most branch in Figure \ref{fig:local-closure} and corresponds to eliminating the middle branch in Figure \ref{fig:local-closure2}. All remaining branches are closed, implying that the case $r$ results in a closure. Such a closure of a case will be called a \emph{local tableau closure}.

The key to identify an inconsistent case, i.e., a local tableau closure, is by checking whether all alternatives implied by the propositions $\hat{\mathcal{S}}$ of a closed branch with support $\mathcal{S}$ for the closure, are also closed. Consider the closed left branch in Figures \ref{fig:local-closure} and \ref{fig:local-closure2}. The support $\mathcal{S} = \{ \neg(p \wedge q), (\{ r \vee s \}, r \leadsto p) \}$ for the closure is based on one of the two cases implied by $\neg(p \wedge q)$, namely the case in which $\neg p$ holds. It is possible that the other case in which $\neg q$ holds, is consistent. The case $r$ implied by $r \vee s$ can only be inconsistent if both sub-case $p$ and $q$ are inconsistent.

To determine whether a case is inconsistent; i.e., whether we have a local tableau closure, we need to consider all cases implied by a set of propositions $\hat{\mathcal{S}}$ where $\mathcal{S}$ is the support of a branch closure. Since these cases can be spread over the whole tableau, we will propagate the support for branch closures towards the root of the tableau. Cases are the result of applying tableau rules that create more than one child node. We can therefore combine the supports for closures of sub-cases at nodes with more than one child node while propagating the supports for branch closures towards the root. This procedure enables us to check for the propositions involved in a leaf closure whether all cases implied by these propositions are closed.

\begin{definition} \label{def:local-closure}
	Let $\mathcal{T}$ be an argumentation tableau with root $\Gamma_0$ and with leaf nodes: $\Lambda_1, \ldots, \Lambda_n$. Moreover, let $\Lambda_{i_1}, \ldots, \Lambda_{i_k}$ be  the closed leaf nodes. We propagate the support for the closure of a leaf toward the root of the tableau.
	\begin{itemize}
		\item
		If the argument $(\mathcal{S},\eta)$ was rewritten in a node $\Gamma$ and resulted in one child node $\Gamma'$, then add every $(\mathcal{S},\bot) \in \Gamma'$ to $\Gamma$.
		\item
		If the argument $(\mathcal{S},\eta)$ was rewritten in a node $\Gamma$ and resulted in more than one child node $\Gamma_1,\ldots,\Gamma_m$, then add every $(\bigcup_{i=1}^m \mathcal{S}_i, \bot)$ with $(\mathcal{S}_i,\bot) \in \Gamma_i$ and $\mathcal{S} \subseteq \mathcal{S}_i$, to $\Gamma$.
		\item
		If the argument $(\mathcal{S},\eta)$ was rewritten in a node $\Gamma$ and resulted in more than one child node $\Gamma_1,\ldots,\Gamma_m$, then add every $(\mathcal{S}_i,\bot) \in \Gamma_i$ such that $\mathcal{S} \not\subseteq \mathcal{S}_i$, to $\Gamma$.
	\end{itemize}
	Every  $(\mathcal{S},\bot) \in \Gamma_0$ represents a local tableau closure.
\end{definition}

When we apply the procedure in this definition to the above example, we get the tableau shown in Figure \ref{fig:support-local-closure}. The tableau supports the local closure that we expect. 

\begin{figure}
\scriptsize\noindent
\Tree
[.{$(\{\neg(p \wedge q)\},\neg(p \wedge q))$ \\
	$(\{ r \vee s \}, r \vee s)$ \\
	$(\{ t \}, t)$ \\
	$\left(\left\{ \begin{array}{l} \neg(p \wedge q), \\ (\{ r \vee s \}, r \leadsto p), \\ (\{ t \}, t \leadsto q) \end{array} \right\}, \bot \right)$}
	[.{$(\{\neg(p \wedge q)\},\neg p)$ \\
		$\boldsymbol{(\{\neg t? \}, \neg t)}$ \\
		$\left(\left\{ \begin{array}{l} \neg(p \wedge q), \\ (\{ r \vee s \}, r \leadsto p) \end{array} \right\}, \bot \right)$}
		[.{$(\{ (\{ t \}, t \leadsto q) \}, q)$ \\
			$\left(\left\{ \begin{array}{l} \neg(p \wedge q), \\ (\{ r \vee s \}, r \leadsto p) \end{array} \right\}, \bot \right)$}
			[.{$(\{ r \vee s \}, r)$ \\
				$\boldsymbol{(\{\neg r? \}, \neg r)}$ \\
				$\left(\left\{ \begin{array}{l} \neg(p \wedge q), \\ (\{ r \vee s \}, r \leadsto p) \end{array} \right\}, \bot \right)$}
				[.{$(\{ (\{ r \vee s \}, r \leadsto p) \}, p)$ \\
					$\left(\left\{ \begin{array}{l} \neg(p \wedge q), \\ (\{ r \vee s \}, r \leadsto p) \end{array} \right\}, \bot \right)$}
					[.{$\left(\left\{ \begin{array}{l} \neg(p \wedge q), \\ (\{ r \vee s \}, r \leadsto p) \end{array} \right\}, \bot \right)$}
					]
				]
			]
			[.{$(\{ r \vee s \}, s)$}
			]
		]
	]
	[.{$(\{\neg(p \wedge q)\},\neg q)$ \\
		$\boldsymbol{(\{\neg t? \}, \neg t)}$ \\
		$\left(\left\{ \begin{array}{l} \neg(p \wedge q), \\ (\{ t \}, t \leadsto q) \end{array} \right\}, \bot \right)$}
		[.{$(\{ (\{ t \}, t \leadsto q) \}, q)$ \\
			$\left(\left\{ \begin{array}{l} \neg(p \wedge q), \\ (\{ t \}, t \leadsto q) \end{array} \right\}, \bot \right)$}
			[.{$\left(\left\{ \begin{array}{l} \neg(p \wedge q), \\ (\{ t \}, t \leadsto q) \end{array} \right\}, \bot \right)$}
			]
		]
	]
] 
\normalsize
\caption{The support for a local tableau closure.}
\label{fig:support-local-closure}
\end{figure}

We can prove that Definition \ref{def:local-closure} guarantees that supports for local closures represent inconsistent cases.

\begin{proposition} \label{prop:local-closure}
	If $\mathcal{S}$ is the support for the local closures of a tableau, then $\hat{\mathcal{S}} \vdash \bot$.
\end{proposition}
\begin{proof}
	Let $(\mathcal{S}_1,\bot), \ldots, (\mathcal{S}_k,\bot)$ be the closures of the leafs $\Lambda_1, \ldots, \Lambda_k$ that resulted in the support $\mathcal{S}$ according to Definition \ref{def:local-closure}. Consider the propagation of $(\mathcal{S}_1,\bot), \ldots, (\mathcal{S}_k,\bot)$ towards the root of the tableau. 
	\begin{itemize}
		\item
		Each time the third item of Definition \ref{def:local-closure} was applied, remove all branches except for the current branch over which we propagate the closure. The removed side branches do not contribute to the support $\mathcal{S}$. 
		\item
		Next, remove from  all nodes, the elements $(\mathcal{S}',\eta)$ for which $\mathcal{S}' \not\subseteq \mathcal{S}$. 
		\item
		Finally, add $\{ (\{ (\mathcal{S}', \varphi \leadsto \psi)\},\psi) \mid (\mathcal{S}', \varphi \leadsto \psi) \in \mathcal{S}\}$ to every node of the tableaux  to get a proper argumentation tableau. Note that some nodes $\Gamma$ may have children that are all equal to $\Gamma$. 
	\end{itemize}
	The following holds for the resulting tableau $\mathcal{T}'$:
	\begin{itemize}
		\item 
		The tableau $\mathcal{T}'$ has a root $\Gamma_0 = \{ (\{ \sigma \}, \sigma) \mid \sigma \in \mathcal{S}\cap \mathcal{L} \} \cup \{ (\{ (\mathcal{S}', \varphi \leadsto \psi)\},\psi) \mid (\mathcal{S}', \varphi \leadsto \psi) \in \mathcal{S}\}$.
		\item 
		There is an interpretation entailing the root iff it entails every node on a branch from the root to a leaf.
		\item 
		The tableau $\mathcal{T}'$ still closes with $\mathcal{S}$ being the support of a tableau closure according to Definition \ref{def:closure}.
	\end{itemize}
	Hence, $\hat{\mathcal{S}}$ is an inconsistent set of propositions.
	\hspace*{\fill} $\Box$
\end{proof}

We can also prove that inconsistent cases can be identified through supports for local tableau closures.

\begin{proposition} \ \\
Let $\{ (\{\sigma_1\},\sigma_1), \ldots, (\{\sigma_m\},\sigma_m), (\mathcal{S}_1,\eta_1 \leadsto \mu_1),\ldots,(\mathcal{S}_n,\eta_n \leadsto \mu_n) \}$ be a minimal inconsistent case. \\
Then $\mathcal{S} = \{ (\{\sigma_1\},\sigma_1), \ldots, (\{\sigma_m\},\sigma_m), (\mathcal{S}_1,\eta_1 \leadsto \mu_1),\ldots,(\mathcal{S}_n,\eta_n \leadsto \mu_n) \}$ is a support for a local closure.
\end{proposition}
\begin{proof}
Since $\{ \sigma_{1}, \ldots, \sigma_{m}, \mu_{1}, \ldots, \mu_{n} \}$ is a minimal inconsistent set, each branch containing an element of $\mathcal{S}$ can be closed by extending the tableau. The support for each closure of the branches is a subset of $\{ (\{\sigma_1\},\sigma_1), \ldots, (\{\sigma_m\},\sigma_m), (\mathcal{S}_1, \\ \eta_1 \leadsto \mu_1),\ldots,(\mathcal{S}_n,\eta_n \leadsto \mu_n) \}$. We can propagate the supports towards the root as specified by Definition \ref{def:local-closure}. Since all sub-cases are closed, the propagation will be successful and the root will have a support $\mathcal{S}'$ for the local closure. 
	
	Suppose that $\mathcal{S}' \not= \mathcal{S}$. Then $\mathcal{S}' \subset \mathcal{S}$ and according to Proposition \ref{prop:local-closure}, $\hat{\mathcal{S}}'$ is an inconsistent set implying that $\hat{\mathcal{S}}$ is not a minimal inconsistent set. Contradiction.
	
	Hence, $\mathcal{S}$ is a support for a local tableau closure.
	\hspace*{\fill} $\Box$
\end{proof}

\subsection{Mutually exclusive cases}
There is one last issue concerning reasoning by cases. The tableau rule $\frac{(\mathcal{S},\varphi \vee \psi)}{(\mathcal{S},\varphi) \mid (\mathcal{S},\psi)}$ does not guarantee that cases are mutually exclusive.\footnote{Note that the goal is not to define a tableau rule for an `exclusive or' but for a standard `or', which can be viewed as describing three mutually exclusive cases.}
The applying this tableau rule results in two children representing two cases. Both cases may support a conclusion $\eta$. This conclusion is not justified if $\eta$ does not hold when both $\varphi$ and $\psi$ are true. As an illustration, suppose that a party will be great if Harry or Ron will attend it, but not if both will attend (because Harry and Ron have a quarrel). 
Here, the case that Harry attends the party and whether Ron attends is unknown, is not the same as drawing a conclusion in the absence of more specific information. The disjunction implies that Ron might attend the party too. The solution to this issue is to ensure that the tableau only contains cases that are mutually exclusive.
We address this problem by adapting three tableau rules.

\begin{center}
	\begin{tabular}{cc}
		\small$\displaystyle \frac{(\mathcal{S},\varphi \vee \psi)}{(\mathcal{S},\varphi \wedge \neg \psi) \mid (\mathcal{S},\varphi \wedge \psi) \mid (\mathcal{S},\neg \varphi \wedge \psi)} $ & 
		\small$\displaystyle \frac{(\mathcal{S},\varphi \to \psi)}{(\mathcal{S},\neg \varphi \wedge \neg \psi) \mid (\mathcal{S},\neg \varphi \wedge \psi) \mid (\mathcal{S},\varphi \wedge \psi)} $ 
		\\
		\\
		\small$\displaystyle \frac{(\mathcal{S},\neg(\varphi \wedge \psi))}{(\mathcal{S},\neg\varphi \wedge \psi) \mid (\mathcal{S},\neg \varphi \wedge \neg \psi) \mid (\mathcal{S},\varphi \wedge \neg \psi)} $
	\end{tabular}
\end{center}

Using these adapted tableau rules we will consider three mutually exclusive cases given the information that Harry or Ron will attend the party. In two cases the party will be great and in one case it will not. 

\section{Conclusion} \label{sec:conclusion}
This paper investigated the possibility of using the semantic tableau method to derive arguments for claims / conclusions. We conclude that it is possible to define an argumentation tableau that provides the arguments supporting conclusions in case of propositional and predicate logic. If the initial information is inconsistent, undercutting arguments can also be derived for resolving the inconsistencies. 
We further conclude that an argumentation tableau can provide arguments supporting conclusions if propositional and predicate logic are extended with defeasible rules. Arguments for inconsistencies, covering rebutting attacks, can be resolved by deriving undercutting arguments for defeasible rules. 
Our last conclusion is that an argumentation tableau enables reasoning by cases and that conclusions supported by reasoning by cases are intuitively plausible.

Further research can be done on (\textit{i}) efficiently implementing an argumentation tableau, and (\textit{ii}) adapting the argumentation tableau to other logics.

\bibliographystyle{splncs04}
\bibliography{NMR,tableau}

\end{document}